\documentclass[11pt]{article}

\newlength{\defbaselineskip}
\usepackage{graphicx}
\usepackage{amsmath}
\usepackage{amssymb}
\usepackage{algorithm}
\usepackage[noend]{algpseudocode} 
\usepackage[caption=false,font=footnotesize]{subfig}

\newtheorem{thm}{Theorem}
\newtheorem{proof}{Proof}

\newcommand{\x}{\mathbf{x}}

\newcommand{\X}{\mathbf{X}}
\newcommand{\C}{\mathbf{C}}
\newcommand{\R}{\mathbf{R}}
\newcommand{\E}{\mathbb{E}}    
\newcommand{\RR}{\mathbb{R}}    
\DeclareMathOperator{\diag}{diag}
\DeclareMathOperator{\tr}{tr}

\begin{document}

\title{
Estimation of the sample covariance matrix from compressive measurements\footnote{This paper is a postprint of a paper submitted to and accepted for publication in \textit{IET Signal Processing} and is subject to Institution of Engineering and Technology Copyright. The copy of record is available at the IET Digital Library.}
}

\author{
Farhad Pourkamali-Anaraki
\thanks{
Department of Electrical, Computer, and Energy Engineering,
University of Colorado Boulder,
Boulder, CO, USA.
E-mail:  farhad.pourkamali@colorado.edu
}
}


\maketitle

\begin{abstract}
	This paper focuses on the estimation of the sample covariance matrix from low-dimensional random projections of data known as compressive measurements. In particular, we present an unbiased estimator to extract the covariance structure from compressive measurements obtained by a general class of random projection matrices consisting of i.i.d.~zero-mean entries and finite first four moments. In contrast to previous works, we make no structural assumptions about the underlying covariance matrix such as being low-rank. In fact, our analysis is based on a non-Bayesian data setting which requires no distributional assumptions on the set of data samples.
	Furthermore, inspired by the generality of the projection matrices, we propose an approach to covariance estimation that utilizes sparse Rademacher matrices. Therefore, our algorithm can be used to estimate the covariance matrix in applications with limited memory and computation power  at the acquisition devices. 
	Experimental results demonstrate that our approach allows for accurate estimation of the sample covariance matrix on several real-world data sets, including video data. 
\end{abstract}
\section{Introduction}\label{intro}
In recent years, there has been growing interest in developing signal processing tasks and learning methods that work directly on compressive versions of data~\cite{DavenportCSSignal}. The main idea behind this approach is to maintain informative sketches by using low-dimensional random projections of data known as compressive measurements. The task of learning is then done using the compressive measurements instead of the full data. As a result, these learning techniques can be used to reduce the computational, storage, and communication burden in  today's high-dimensional data regimes~\cite{BigDataSigMagazine}. 

There are several lines of work that focus on performing signal processing tasks and information retrieval using only the partial information embedded in the compressive measurements. For example, the authors in~\cite{PerformanceLimitsCS} studied the performance limit of classification based on compressive measurements. In~\cite{ChangeAtia}, Atia studied the problem of change detection based on random 
projections that satisfy the restricted isometry property (RIP)~\cite{CSCandes}.
Moreover, data-adaptive dictionary learning from compressive versions of data has been studied in several works, including~\cite{BlindCS_2010,StuderDL,Pourkamali_CKSVD,Pourkamali_DL_SampTA}.  

In this paper, we focus on the  estimation of the sample covariance matrix from compressive measurements. Covariance matrix estimation is a fundamental problem in signal
processing and data analysis. For example, eigendecomposition of the data's covariance matrix can be used to perform principal component analysis (PCA), which is frequently used for feature extraction~\cite{MallatPCA}.

In this work, we propose and analyze an unbiased estimator for the sample covariance matrix, where the data samples are observed only through low-dimensional random projections. In particular, we consider a general class of random projections that has entries drawn independent and identically distributed (i.i.d.) from a zero-mean distribution with finite first four moments. Thus, one can use sparse random projections~\cite{VerySparseRandom} with independent Rademacher entries (a Rademacher random variable takes two values $\pm1$ with equal probability) to achieve reduction of the data acquisition cost and computational complexity. Therefore, our covariance estimator can be utilized on power constrained mobile devices and wearables such as cell phones and surveillance cameras~\cite{BigDataSigMagazine,sindhwani2015structured}. Our estimator can also be used in the distributed covariance estimation problem, e.g., environmental/atmospheric monitoring, where the data samples are observed at distributed sensors. In such settings, we would like to reveal the underlying structure with minimal measurement and communication overhead~\cite{PETRELS,kumar2013survey,azizyan2014subspace}.

The key contributions that differentiate this paper from prior work are: (i) a fixed data setting is considered without making distributional assumptions on the set of data samples; (ii) we do not impose structural assumptions on the covariance matrix such as being low-rank; and (iii) our theoretical analysis does not require random projections that satisfy the distributional Johnson-Lindenstrauss property~\cite{park2014modal}. Instead, our theoretical result applies to a wide range of random projections with entries drawn from a zero-mean distribution.

The problem setup and a brief review of prior work are presented in Section~\ref{Preliminaries}. In Section~\ref{CovEstimator}, we introduce our unbiased estimator for the sample covariance matrix, with theoretical guarantees deferred to Section~\ref{Proof}. In Section~\ref{ExperResults}, we present experimental results on three real-world data sets from different application domains to verify the performance of our covariance estimator in the compressive domain.
\section{Preliminaries}\label{Preliminaries}
\subsection{Notation}\label{Notation}
We name column vectors by lower-case bold letters and matrices by upper-case bold letters. The matrix $\mathbf{I}_{p\times p}$ is the identity matrix of size $p$. For a symmetric matrix $\C\in\RR^{p\times p}$, $\|\C\|_2$ represents the spectral norm and defined as $\|\C\|_2=\max_{\mathbf{v}:\|\mathbf{v}\|_2=1}\mathbf{v}^T\C\mathbf{v}$.

Moreover, let $\tr(\C)$ denote the trace of the matrix $\C$, and $\diag(\C)$ represent the matrix formed by zeroing all but the diagonal entries of the matrix $\C$. 
\subsection{Problem formulation}\label{Formulation}
Consider a collection of data samples $\x_1,\ldots,\x_n$ in $\RR^{p}$ and let $\X=[\x_1,\ldots,\x_n]\in\RR^{p\times n}$ be a matrix whose $i$-th column is the data sample $\x_i$. In this work, we do not make any distributional assumptions on the data and our goal is to recover the sample covariance matrix $\C_n$ from a single pass over the data:
\begin{equation}\label{eq:C_n}
	\C_n=\frac{1}{n}\sum_{i=1}^{n}\x_i\x_i^{T}=\frac{1}{n}\X\X^{T}.
\end{equation}
In modern data settings, the high dimensionality of massive data sets and limited resources in the data acquisition process make full data access infeasible~\cite{Fowler}. To address this problem, we aim to recover $\C_n$ from compressive measurements of the data. This can be viewed as forming a random matrix $\R_i\in \RR^{p\times m}$ with $m<p$ for  $\x_i$, $i=1,\ldots,n$. The entries of $\R_i$ are drawn i.i.d.~from a zero-mean distribution with finite first four moments~\cite{NewBoundsCompressive}. We are interested in estimating the sample covariance matrix $\C_n$ from low-dimensional random projections $\{\R_i^{T}\x_i\}_{i=1}^{n}\in \RR^m$. 

As we will see, the shape of the distribution in generating random matrices $\{\R_i\}_{i=1}^{n}$ is an important factor in order to find an unbiased estimator for the sample covariance matrix. The kurtosis is a measure of heaviness of tail for a distribution. It is defined as $\kappa:=\frac{\mu_4}{\mu_2^2}-3$, where $\mu_2$ and $\mu_4$ are the second and fourth moments of the distribution. The value of kurtosis can also be used as a measure of deviation from the Gaussian distribution, since the kurtosis of a zero-mean Gaussian distribution is zero. 
\subsection{Relation to prior work}\label{PriorWork}
Our work is most closely related to~\cite{Pourkamali_ICASSP_2014,Pourkamali_ICML_2014}, where the estimation of covariance matrices from low-dimensional random projections is studied. It was shown that the sample covariance matrix of $\{\R_i\R_i^{T}\x_i\}_{i=1}^{n}$ (scaled by a known factor) can be used to estimate the covariance structure of the full data $\{\x_i\}_{i=1}^{n}$.  The motivation behind this technique is that $\R_i\R_i^{T}\x_i$ approximates the projection of $\x_i$ onto the column space of $\R_i\in\RR^{p \times m}$, where the expensive matrix inverse is eliminated in $\R_i(\R_i^{T}\R_i)^{-1}\R_i^{T}\x_i$.

A major drawback of the prior work is that the result only holds for data samples drawn from a specific probabilistic model known as the spiked covariance model. In fact, this simple statistical model is popular in studying the covariance estimation problem, but unfortunately many real-world data sets do not fall into the spiked covariance model. 

Another disadvantage of the previous work is that the proposed estimators are biased. One must have prior knowledge of the parameters of the probabilistic model to modify these estimators. In reality, it is almost infeasible to accurately estimate these parameters.
\section{The proposed unbiased covariance estimator}\label{CovEstimator}
In this section, we present an unbiased estimator for the sample covariance matrix of the full data $\C_n=\frac{1}{n}\sum_{i=1}^{n}\x_i\x_i^T$ from the compressive measurements $\{\R_i^{T}\x_i\}_{i=1}^{n}\in\RR^{m}$. Recall that the entries of $\R_i\in \RR^{p\times m}$, $m<p$, are drawn i.i.d.~from a zero-mean distribution with finite second and fourth moments $\mu_2$ and $\mu_4$, and kurtosis $\kappa$. One advantage of our approach compared to the prior work is that our work makes much weaker assumptions on the data.  Therefore, our analysis can be used to develop a better understanding of extracting the covariance structure in the compressive domain. 
\subsection{Compressive covariance estimation}\label{CCE}
To begin, let's consider a rescaled version of the sample covariance matrix of the projected data $\{\R_i\R_i^{T}\x_i\}_{i=1}^{n}\in \RR^{p}$:
\begin{equation}
	\widehat{\C}_n\!:=\!\frac{1}{(m^2+m)\mu_2^2}\!\cdot\!\frac{1}{n}\sum_{i=1}^{n}\R_i\R_i^{T}\x_i\x_i^{T}\R_i\R_i^{T}.\label{eq:C_hat_n}
\end{equation}
We compute the expectation of $\widehat{\C}_n$ based on Theorem~\ref{Lemma_Random_Matrix} in Section~\ref{Proof}:
\begin{equation}
	\E[\widehat{\C}_n]\!=\!\C_n\!+\!\frac{\kappa}{m+1}\diag\left(\C_n\right)\!+\!\frac{1}{m+1}\!\tr\left(\C_n\right)\!\mathbf{I}_{p\times p}\label{eq:exp-C-hat-n}
\end{equation}
where $\diag(\C_n)$ represents the matrix formed by zeroing all but the diagonal entries of $\C_n$.

We observe that the expectation of the covariance estimator $\widehat{\C}_n$ has three components: the sample covariance matrix of the full data $\C_n$ that we wish to extract and two additional bias terms. The bias term $\frac{\kappa}{m+1}\diag(\C_n)$ can be viewed as representing a bias of the estimator towards the nearest canonical basis vectors. The value of this term depends on both the kurtosis of the distribution used in generating random matrices $\{\R_i\}_{i=1}^{n}$ and the structure of the underlying covariance matrix $\C_n$. In contrast, the bias term $\frac{1}{m+1}\tr(\C_n)\mathbf{I}_{p\times p}$ in~\eqref{eq:exp-C-hat-n} is independent of the shape of the distribution. This term shows that the energy of the compressed data is somewhat scattered into different directions in $\RR^{p}$, as $\tr(\C_n)$ represents the energy of the input data. Note that these two bias terms are decreasing functions of $m$, where $m$ is the number of linear measurements. 

The challenge here is to remove these two bias terms by modifying the sample covariance matrix $\widehat{\C}_n$ in the compressive domain. To do this, we first find the expectation of $\diag(\widehat{\C}_n)$ by using linearity of expectation:
\begin{eqnarray}
	& \E[\diag(\widehat{\C}_n)] & \!\!\!=\diag(\E[\widehat{\C}_n])\nonumber\\
	&  &\!\!\!= \left(1\!+\!\frac{\kappa}{m+1}\right)\!\diag(\C_n)\!+\!\frac{\tr\left(\C_n\right)}{m+1}\mathbf{I}_{p\times p}.
\end{eqnarray}
We also need to compute the expectation of $\tr(\widehat{\C}_n)$:
\begin{equation}
	\E[\tr(\widehat{\C}_n)]=\tr(\E[\widehat{\C}_n])=\frac{(m+1+\kappa+p)}{m+1}\tr(\C_n)
\end{equation}
since trace is a linear operator. Hence, we can modify $\widehat{\C}_n$ to obtain an unbiased estimator for the sample covariance matrix of the full data:
\begin{equation}
	\widehat{\boldsymbol{\Sigma}}_n:=\widehat{\C}_n-\alpha_1\diag\left(\widehat{\C}_n\right)-\alpha_2\tr\left(\widehat{\C}_n\right)\mathbf{I}_{p\times p}\label{eq:Sigma_hat_n}
\end{equation}
where 
\begin{equation}\label{eq:alpha1}
	\alpha_1:=\frac{\frac{\kappa}{m+1}}{1+\frac{\kappa}{m+1}}
\end{equation}
and 
\begin{equation}\label{eq:alpha2}
	\alpha_2:=\frac{1}{(1+\frac{\kappa}{m+1})(m+1+\kappa+p)}.
\end{equation}
We immediately see that $\widehat{\boldsymbol{\Sigma}}_n$ is an unbiased estimator for the sample covariance matrix of the full data:
\begin{equation}
	\E[\widehat{\boldsymbol{\Sigma}}_n]=\C_n.
\end{equation}
Note that $\alpha_1$ and $\alpha_2$ are two \textit{known} constants that are used to modify the biased estimator $\widehat{\C}_n$ based on $\diag(\widehat{\C}_n)$ and $\tr(\widehat{\C}_n)$.
\subsection{Covariance estimation via sparse random projections}\label{CCE_sparse}
Throughout our analysis, the entries of $\{\R_i\}_{i=1}^{n}$ are drawn from a zero-mean distribution with finite first four moments. A common choice is a random matrix with i.i.d.~Gaussian entries $N(0,1)$ as in~\cite{Pourkamali_ICASSP_2014}. In this case, the kurtosis of the Gaussian distribution is zero which means that the bias term $\frac{\kappa}{m+1}\diag(\C_n)=\textbf{0}_{p\times p}$, regardless of the structure of $\C_n$. In fact, this is due to the rotational invariance of the Gaussian distribution which eliminates the dependence on the orientation of data. However, as the random Gaussian matrix is dense, the memory and computational loads are high for large-scale data sets.

In this paper, we adopt sparse random projections that were originally proposed to reduce the memory and computational burden in the estimation of pairwise distances~\cite{VerySparseRandom}. In this setting, the entries of $\{\R_i\}_{i=1}^{n}\in\RR^{p\times m}$ are distributed on $\{-1,0,+1\}$ with probabilities $\{\frac{1}{2s},1-\frac{1}{s},\frac{1}{2s}\}$. Therefore, the parameter $s\geq1$ controls the sparsity of random matrices such that each column of $\R_i$ has $\frac{p}{s}$ non-zero entries, on average. It is easy to verify that the second and fourth moments are $\mu_2=\mu_4=\frac{1}{s}$ and thus the kurtosis:
\begin{equation}
	\kappa=\frac{\mu_4}{\mu_2^2}-3=s-3.
\end{equation}
Hence, as the sparsity of $\{\R_i\}_{i=1}^{n}$ increases, the value of the bias term $\frac{\kappa}{m+1}\diag(\C_n)=\frac{s-3}{m+1}\diag(\C_n)$ in spectral norm increases proportional to the sparsity parameter $s$. This means that the modification of the covariance estimator $\widehat{\C}_n$ becomes more important as the sparsity parameter $s$ increases. 

In this work, we are specifically interested in choosing the parameters $s$ and $m$ such that the \textit{compression factor}:
\begin{equation}
	\gamma:=\frac{m}{s}<1. 
\end{equation}
The benefits of this choice are twofold. First, the average computation cost to acquire or access each data sample is $O(\frac{mp}{s})=O(\gamma p)$, $\gamma<1$, compared to the cost for acquiring the full data sample $O(p)$. Hence, we can achieve a substantial cost reduction in the data acquisition process when the number of samples $n$ is large. 

Second, the expected number of non-zero entries in each of the $m$ columns of $\R_i\in\RR^{p\times m}$ is $\frac{p}{s}$. Therefore, each projected data $\R_i\R_i^T\x_i\in \RR^p$ has at most $O(\frac{mp}{s})=O(\gamma p)$ non-zero entries on average. Thus, the cost to form $(\R_i\R_i^T\x_i)(\R_i\R_i^T\x_i)^T$ will be $O(\gamma^2p^2)$,  in contrast to the cost for computing $\x_i\x_i^T$ which is $O(p^2)$. Hence, the computational complexity to form the sample covariance matrix $\widehat{\C}_n$ in the compressive domain is reduced by a factor of $\gamma^2$ for some $\gamma<1$. 
Empirical results will be provided later in Section~\ref{ExperResults} to reveal the tradeoffs between computational savings and accuracy on various data sets.

\section{Theoretical results}\label{Proof}
In this section, we provide theoretical guarantees on the expectation of  $\widehat{\C}_n$, which is introduced in the previous section.
\begin{thm}\label{Lemma_Random_Matrix}
	Consider a rescaled version of the sample covariance matrix in the compressive domain:
	\begin{equation}\label{eq:Thm_C_hat_n_1}
		\widehat{\C}_n=\frac{1}{(m^2+m)\mu_2^2}\cdot\frac{1}{n}\sum_{i=1}^{n}\R_i\R_i^{T}\x_i\x_i^{T}\R_i\R_i^{T}
	\end{equation}
	where $\{\x_i\}_{i=1}^{n}\in\RR^p$ are data samples and the entries of $\{\R_i\}_{i=1}^{n}\in\RR^{p\times m}$, $m < p$ , are drawn i.i.d.~from a zero-mean distribution with finite second and fourth moments $\mu_2$ and $\mu_4$, and kurtosis $\kappa$. Then, the expectation of $\widehat{\C}_n$ over the randomness in $\{\R_i\}_{i=1}^{n}$:
	\begin{equation}\label{eq:Thm_C_hat_n}
		\E[\widehat{\C}_n]\!=\!\C_n\!+\!\frac{\kappa}{m+1}\!\diag\left(\C_n\right)\!+\!\frac{1}{m+1}\!\tr\left(\C_n\right)\!\mathbf{I}_{p\times p}
	\end{equation}
	where $\C_n=\frac{1}{n}\sum_{i=1}^{n}\x_i\x_i^{T}$ and $\diag(\C_n)$ denotes the matrix formed by zeroing all but the diagonal entries of $\C_n$.
\end{thm}
\begin{proof}
	First, we compute the expectation of each summand in~\eqref{eq:Thm_C_hat_n_1}. Note that the authors in~\cite{Pourkamali_ICML_2014} computed this expectation when each data sample $\x_i$ is drawn i.i.d.~from the spiked covariance model. In this statistical model, it is assumed that the data samples are drawn from a low-rank Gaussian distribution. Here, a new analysis is given where we make no distributional assumptions on the data samples to provide a better understanding of covariance estimation from compressive measurements. 
	
	Consider any $\R_i$, which we call $\R$ to simplify the notation. The entry in the $i$-th row and $j$-th column of $\R\in\RR^{p\times m}$ is denoted by $r_{ij}$, where $1\leq i\leq p$ and $1\leq j\leq m$. Thus, the $k$-th column of the matrix $\R\R^{T}$ has the following form:
	\[
	\mathbf{c}_{k}\!=\![\begin{array}{ccccc}
	\!\sum_{j=1}^{m}r_{1j}r_{kj}, & \!\! \ldots, &\!\!\!\! \sum_{j=1}^{m}r_{kj}^{2},&\!\! \ldots, & \!\!\!\sum_{j=1}^{m}r_{pj}r_{kj}\!\!\end{array}]^{T}.
	\]
	Define $\mathbf{E}_{k,l}:=\E[\mathbf{c}_{k}\mathbf{c}_{l}^{T}]$. We calculate the entries of the matrix $\mathbf{E}_{k,k}=\mathbb{E}[\mathbf{c}_{k}\mathbf{c}_{k}^{T}]$
	using the fact that the entries of $\R$ are drawn i.i.d.~from a zero-mean distribution:\\
	(1) The $k$-th diagonal entry of $\mathbf{E}_{k,k}$: 
	\[
	\mathbb{E}[(\sum_{j=1}^{m}r_{kj}^{2})^{2}]=\mathbb{E}[\sum_{j=1}^{m}r_{kj}^{4}\!+\!\sum_{i\neq j}r_{ki}^{2}r_{kj}^{2}]\!=\!m\mu_{4}\!+\!(m^{2}-m)\mu_{2}^{2}.
	\]
	(2) The other diagonal entries of $\mathbf{E}_{k,k}$, e.g., the $i$-th
	diagonal entry for $i\neq k$: 
	\begin{eqnarray}
		& \mathbb{E}[(\sum_{j=1}^{m}r_{ij}r_{kj})^{2}]& \!\!\!\!= \mathbb{E}[\sum_{j=1}^{m}r_{ij}^{2}r_{kj}^{2}]\!+\!\mathbb{E}[\sum_{j\neq l}r_{ij}r_{kj}r_{il}r_{kl}]\nonumber \\
		& & \!\!\!\!= \mathbb{E}[\sum_{j=1}^{m}r_{ij}^{2}r_{kj}^{2}]\!=\!m\mu_{2}^{2}.\nonumber
	\end{eqnarray}
	(3) The off-diagonal entries of $\mathbf{E}_{k,k}$, e.g.~the entry
	in the $i$-th row and the $l$-th column when $i\neq l$: 
	\[
	\mathbb{E}[(\sum_{j=1}^{m}r_{ij}r_{kj})(\sum_{j=1}^{m}r_{lj}r_{kj})]=0
	\]
	since these entries have at least one term with degree $1$. Therefore, we see that: 
	\begin{eqnarray}
		\mathbf{E}_{k,k} \!&\!\!\! =\!\!\! & \left(m\mu_{4}+\left(m^{2}-2m\right)\mu_{2}^{2}\right)\mathbf{e}_{k}\mathbf{e}_{k}^{T} + m\mu_{2}^{2}\mathbf{I}_{p\times p}\nonumber \\
		\!& \!\!\!=\!\!\! & m\mu_{2}^{2}\left(\frac{\mu_{4}}{\mu_{2}^{2}}-3+m+1\right)\mathbf{e}_{k}\mathbf{e}_{k}^{T}+m\mu_{2}^{2}\mathbf{I}_{p\times p}\nonumber \\
		\!& \!\!\!=\!\!\! & m\mu_{2}^{2}\left(\kappa+m+1\right)\mathbf{e}_{k}\mathbf{e}_{k}^{T}+m\mu_{2}^{2}\mathbf{I}_{p\times p}\label{eq:Ekk}
	\end{eqnarray}
	where $\mathbf{e}_k$ denotes the $k$-th vector of the canonical basis in $\RR^p$. Thus, entries of $\mathbf{e}_k$ are all zero except for the $k$-th one which is $1$.
	
	We can compute the entries of the matrix $\mathbf{E}_{k,l}=\mathbb{E}[\mathbf{c}_{k}\mathbf{c}_{l}^{T}]$ for $k \neq l$ using the same strategy.
	In this case, each entry has at least one term with degree $1$, so
	it will be zero, except the following two terms:\\
	(1) The entry in the $k$-th row and $l$-th column of $\mathbf{E}_{k,l}$: 
	\[
	\mathbb{E}[(\sum_{j=1}^{m}r_{kj}^{2})(\sum_{j=1}^{m}r_{lj}^{2})]=\mathbb{E}[\sum_{j=1}^{m}r_{kj}^{2}]\mathbb{E}[\sum_{j=1}^{m}r_{lj}^{2}]=m^{2}\mu_{2}^{2}
	\]
	where this follows from the fact that the entries of $\R$ are i.i.d.\\
	(2) The entry in the $l$-th row and $k$-th column of $\mathbf{E}_{k,l}$: 
	\[
	\mathbb{E}[(\sum_{j=1}^{m}r_{lj}r_{kj})^{2}]=\mathbb{E}[\sum_{j=1}^{m}r_{lj}^{2}r_{kj}^{2}]=m\mu_{2}^{2}.
	\]
	Hence, we get:
	\begin{equation}
		\mathbf{E}_{k,l}=m^{2}\mu_{2}^{2}\mathbf{e}_{k}\mathbf{e}_{l}^{T}+m\mu_{2}^{2}\mathbf{e}_{l}\mathbf{e}_{k}^{T}.\label{eq:Ekl}
	\end{equation}
	
	Now, we compute the expectation $\E[\R\R^T\x\x^T\R\R^T]$ over the randomness in $\R$ for a fixed data sample $\x=[x_1,\ldots,x_p]^T$:
	\begin{eqnarray}
		\E[\R\R^T\x\x^T\R\R^T]\!& \!\!=\!\!&\!\E[(\sum_{k=1}^{p}x_k\mathbf{c}_k)(\sum_{l=1}^{p}x_l\mathbf{c}_l)^T] \nonumber \\
		\!& \!\! =\!\!&\! \sum_{k=1}^{p}x_k^2\mathbf{E}_{k,k}+\sum_{k\neq l}x_k x_l \mathbf{E}_{k,l} \nonumber \\
		\!&  \!\!=\!\!&\!\left(m^2+m\right)\mu_2^2\x\x^T\!+\!\kappa m\mu_2^2\diag\left(\x\x^T\right)\nonumber\\
		\!& &\!+m\mu_2^2\tr\left(\x\x^T\right)\mathbf{I}_{p\times p} 
	\end{eqnarray}
	where we used $\|\x\|_2^2=\tr(\x\x^T)$. We rescale the expectation:
	\begin{eqnarray}
		& &\frac{1}{\left(m^2+m\right)\mu_2^2}\E[\R\R^T\x\x^T\R\R^T] \nonumber \\
		&&=\x\x^T\!+\!\frac{\kappa}{m+1}\diag\left(\x\x^T\right)\!+\!\frac{1}{m+1}\tr\left(\x\x^T\right)\mathbf{I}_{p\times p}.
	\end{eqnarray}
	Next, we rewrite the sample covariance matrix $\widehat{\C}_n$ as:
	\[
	\widehat{\C}_n=\frac{1}{n}\sum_{i=1}^{n}\frac{1}{\left(m^2+m\right)\mu_2^2}\R_i\R_i^T\x_i\x_i^T\R_i\R_i^T.
	\]
	Using linearity of expectation, it is straightforward to see:
	\begin{equation}
		\E[\widehat{\C}_n]\!=\!\C_n\!+\!\frac{\kappa}{m+1}\diag\left(\C_n\right)\!+\!\frac{1}{m+1}\tr\left(\C_n\right)\mathbf{I}_{p\times p}
	\end{equation}
	and this completes the proof.
\end{proof}

\section{Experimental results}\label{ExperResults}
In this section, we examine the performance of our unbiased covariance estimator $\widehat{\boldsymbol{\Sigma}}_n$ on three real-world data sets. We compare our proposed $\widehat{\boldsymbol{\Sigma}}_n$ with the biased estimator introduced in~\cite{Pourkamali_ICML_2014}. To evaluate the estimate accuracy, we use the normalized covariance estimation error defined as:
\[
\text{normalized estimation error}(\widehat{\boldsymbol{\Sigma}}_n)\!:=\!\frac{\|\widehat{\boldsymbol{\Sigma}}_n\!-\!\C_n\|_2}{\|\C_n\|_2}.
\]
Therefore, this criterion is a measure of closeness of the estimated covariance matrix from compressive measurements to the underlying covariance matrix $\C_n$ based on the spectral norm. Since the estimation of covariance matrix from compressive measurements is stochastic, we re-run each of the following experiments $1000$ times and report the mean.

In all experiments, the parameter $\frac{m}{p}=0.4$ is fixed and we analyze the accuracy of the estimated covariance matrix from the sparse random projections, introduced in Section~\ref{CCE_sparse}, for various values of compression factor $\gamma=\frac{m}{s}$. Note that we are interested in choosing the parameter $s$ such that $\gamma<1$. In fact, a value of $\gamma$ closer to zero indicates the case where we achieve substantial reduction of the acquisition and computation cost by increasing the parameter $s$ in the distribution of random matrices. 
\subsection{MNIST data set}
In the first experiment, we examine the MNIST data set of handwritten digits. This data set consists of centered versions of digits that are stored as a $28\times28$ pixel image. We use data from digit ``0'' which contains $n=6903$ data samples that are vectorized with the dimension $p=28^2=784$. Some examples are shown in Fig.~\ref{fig:mnist_data_visualization}.
\begin{figure}[!h]
	\centering
	\includegraphics[width=0.5\textwidth]{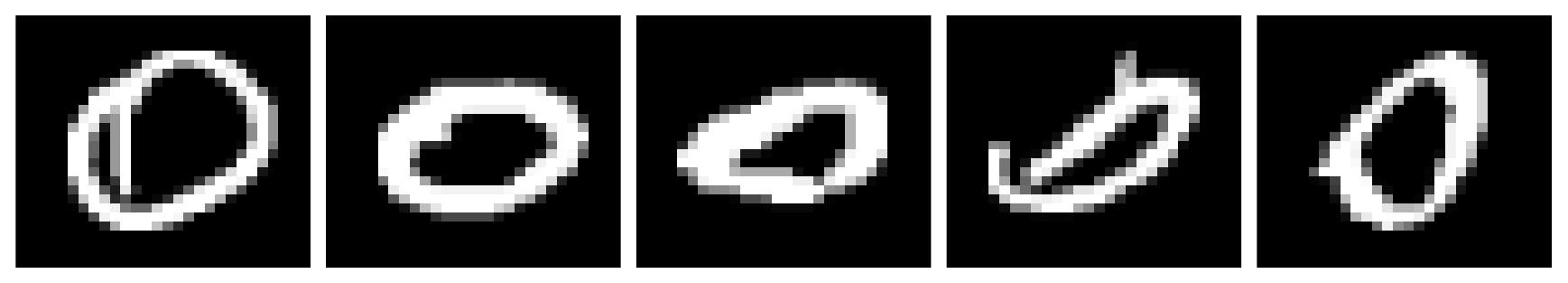}
	\caption{Examples from digit ``0'' in the MNIST data set. 
	}
	\label{fig:mnist_data_visualization}
\end{figure}

The performance of our unbiased estimator $\widehat{\boldsymbol{\Sigma}}_n$ is compared with the biased estimator~\cite{Pourkamali_ICML_2014} in Fig.~\ref{fig:mnist_cov_error}. We see that our proposed estimator $\widehat{\boldsymbol{\Sigma}}_n$ results in more accurate estimates for all values of the compression factor $\gamma$. For example, at $\gamma=0.1$, our estimator $\widehat{\boldsymbol{\Sigma}}_n$ decreases the estimation error by almost a factor of $2$. 
\begin{figure}[!h]
	\centering
	\includegraphics[width=\textwidth]{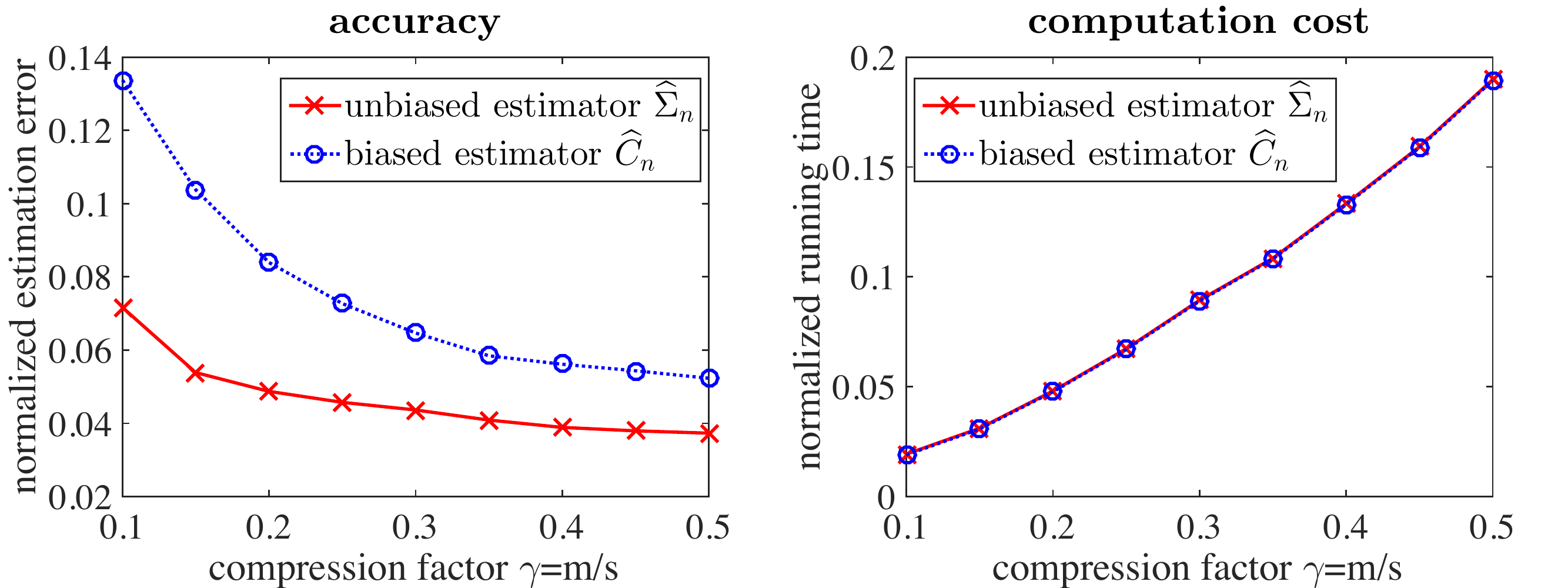}
	\caption{Accuracy and computation cost of the estimated covariance matrix from sparse random projections on the MNIST data set. Our estimator $\widehat{\boldsymbol{\Sigma}}_n$ outperforms the biased estimator introduced in~\cite{Pourkamali_ICML_2014} for all values of the compression factor $\gamma$. 
	}
	\label{fig:mnist_cov_error}
\end{figure}

To see the tradeoffs between computational savings and accuracy, the running times to form our unbiased covariance etsimator $\widehat{\boldsymbol{\Sigma}}_n$ and the biased estimator $\widehat{\C}_n$ are reported in Fig.~\ref{fig:mnist_cov_error}. The running times are normalized by the time spent to form the sample covariance matrix of the full data $\C_n$. Both unbiased and biased estimators from compressive measurements show a speedup over $\C_n$ by almost a factor of $\gamma^2$, as expected from the discussion in Section~\ref{CCE_sparse}. Moreover, we observe that $\widehat{\boldsymbol{\Sigma}}_n$ has roughly the same running time as $\widehat{\C}_n$. Hence, the additional computation cost to remove the bias terms in equation \eqref{eq:Sigma_hat_n} is negligible. Consequently, our unbiased estimator $\widehat{\boldsymbol{\Sigma}}_n$ leads to more accurate estimates than the biased estimator $\widehat{\C}_n$ with approximately the same computation cost.

We also show a sample visual comparison of the first eigenvector of the underlying covariance matrix $\C_n$ and that estimated by our approach $\widehat{\boldsymbol{\Sigma}}_n$ for $\gamma=0.1$, $0.3$, and $0.5$. 
Note that the eigenvectors of covariance matrix encode an orthonormal basis that captures as much of the data's variability as possible. Therefore, we intend to visualize the first eigenvector of $\C_n$ and that estimated by our approach from sparse random projections as a measure of information extraction.
As we see in Fig.~\ref{fig:mnist_pc}, for small values of the compression factor, e.g., $\gamma=0.1$, the resulting eigenvectors from the full data and compressed data are almost identical. In fact, the first eigenvector of our estimator $\widehat{\boldsymbol{\Sigma}}_n$ in Fig.~\ref{fig:mnist_pc} reveals the underlying structure and pattern in the data set consisting of handwritten digits ``0''.
\begin{figure}[!h]
	\centering
	\includegraphics[width=0.5\textwidth]{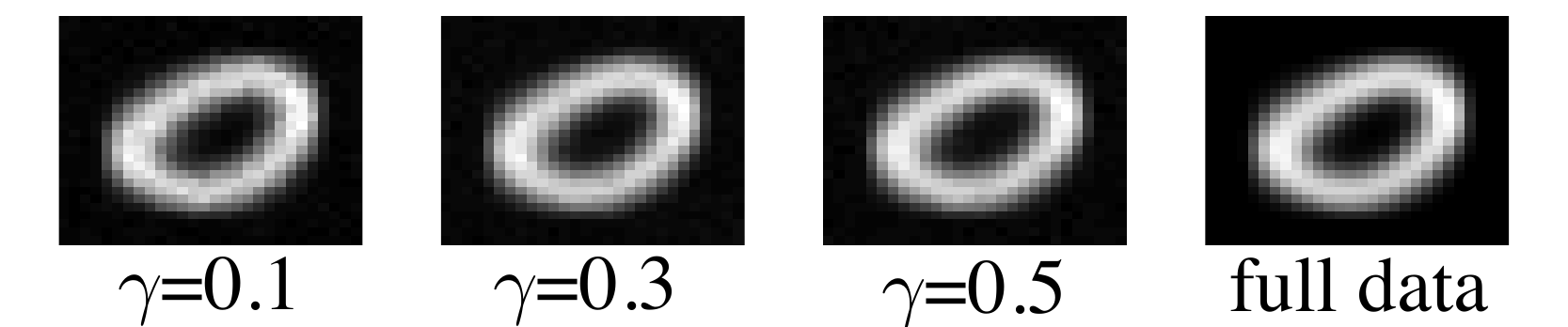}
	\caption{Visual comparison between the first eigenvector of the sample covariance matrix $\C_n$ (full data) and that estimated by our proposed estimator $\widehat{\boldsymbol{\Sigma}}_n$ in the compressive domain for various values of the compression factor $\gamma$. 
	}
	\label{fig:mnist_pc}	
\end{figure}

\subsection{Gen4 data set}
Our second experiment is with the gen4 data matrix of size $1537\times 4298$ ($p=1537$ and $n=4298$) from linear programming problems. This data set is available from the University of Florida Sparse Matrix Collection.

In Fig.~\ref{fig:gen4_cov_error}, the accuracy of our proposed covariance estimator $\widehat{\boldsymbol{\Sigma}}_n$ is compared with the biased estimator in~\cite{Pourkamali_ICML_2014}. We observe that  covariance matrices recovered using our estimator $\widehat{\boldsymbol{\Sigma}}_n$ are significantly more accurate than those returned by~\cite{Pourkamali_ICML_2014}. In fact, our proposed estimator $\widehat{\boldsymbol{\Sigma}}_n$ decreases the estimation error by an order of magnitude.

Moreover, we examine the tradeoffs between computational savings and accuracy on this data set. The running times to compute $\widehat{\boldsymbol{\Sigma}}_n$ and $\widehat{\C}_n$ for various values of $\gamma$ are reported in Fig.~\ref{fig:gen4_cov_error}. Recall that the running times are normalized by the time spent to form the sample covariance matrix of the full data $\C_n$. We see that both estimators show a speedup over $\C_n$ by almost a factor of $\gamma^2$. Similar to the previous example, the additional computation cost to remove the bias terms in equation~\eqref{eq:Sigma_hat_n}  is negligible. Therefore, our unbiased estimator $\widehat{\boldsymbol{\Sigma}}_n$ results in more accurate estimates than $\widehat{\C}_n$ with roughly the same computation cost.
\begin{figure}[!h]
	\centering
	\includegraphics[width=\textwidth]{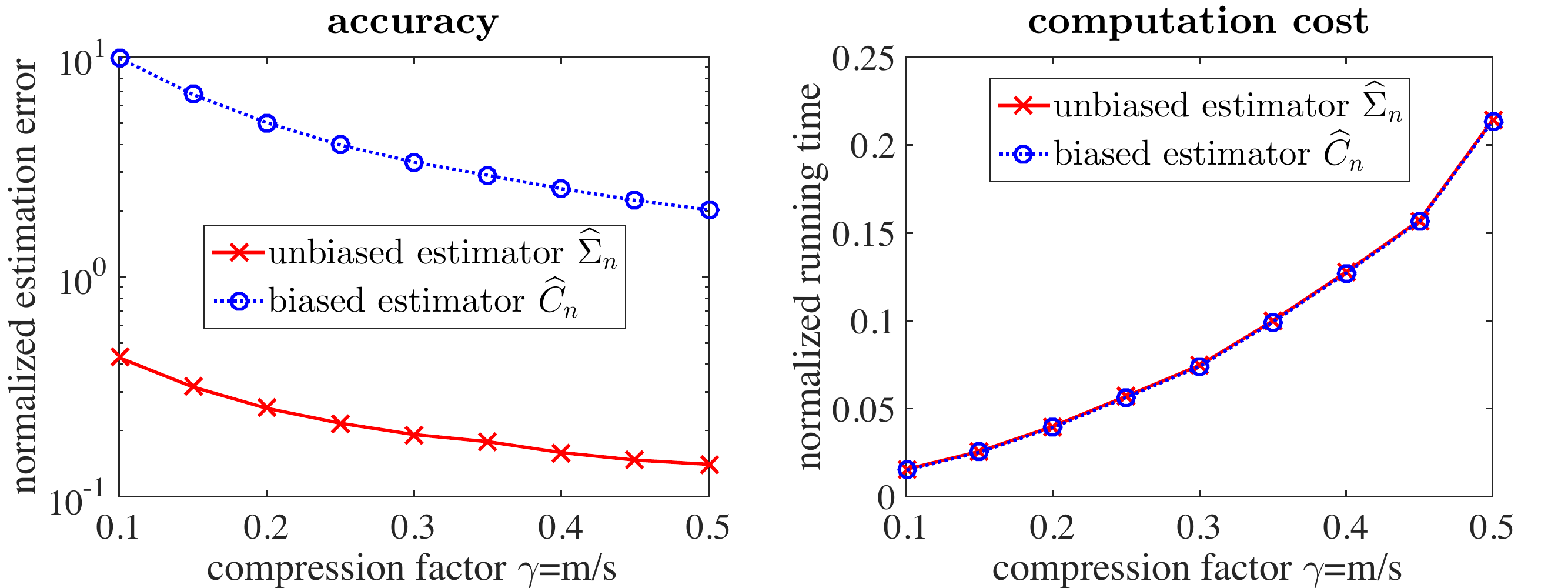}
	\caption{Plot of the normalized covariance estimation error (log scale) and computation cost on the gen4 data set. Our estimator $\widehat{\boldsymbol{\Sigma}}_n$ is compared with the biased estimator in~\cite{Pourkamali_ICML_2014} for varying values of $\gamma$. We see that the unbiased covariance estimator $\widehat{\boldsymbol{\Sigma}}_n$ decreases the estimation error by an order of magnitude. 
	}
	\label{fig:gen4_cov_error}
\end{figure}

\subsection{Traffic data set}
In the last experiment, we consider the traffic data set which contains video surveillance of traffic from a stationary camera. Some examples of this data set are shown in Fig.~\ref{fig:traffic_data_visualization}. Considering individual frames of video as data samples, each frame is a $48\times 48$ pixel image that is vectorized so $p=48^2=2304$. Moreover, we have access to $n=5139$ frames. 
\begin{figure}[!h]
	\centering
	\includegraphics[width=0.5\textwidth]{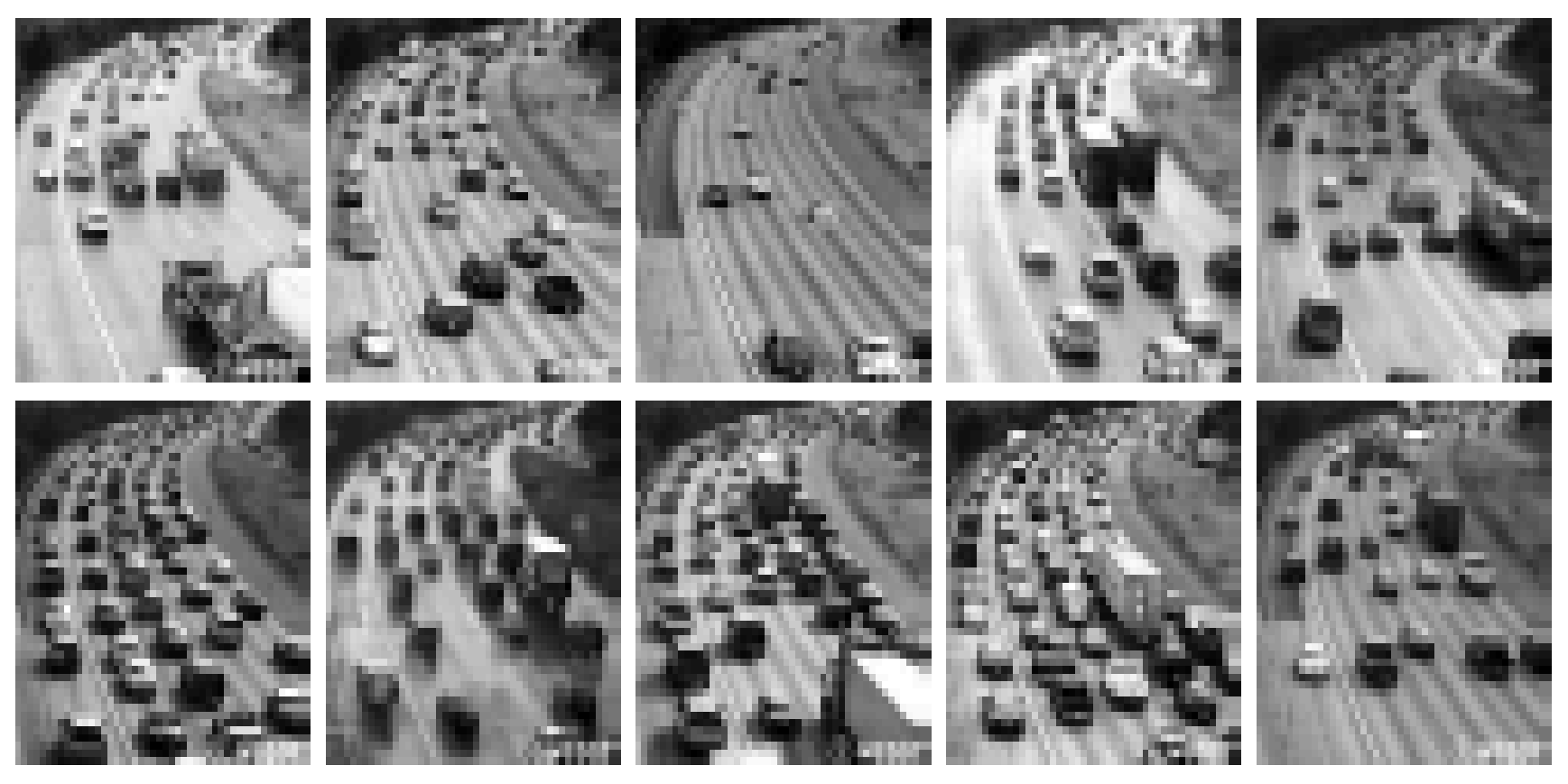}
	\caption{Example frames of the traffic data set.
	}
	\label{fig:traffic_data_visualization}
\end{figure}

Our proposed approach is used to estimate the sample covariance matrix of the full data $\C_n$ from sparse random projections. As we see in Fig.~\ref{fig:traffic_cov_error}, using unbiased estimator $\widehat{\boldsymbol{\Sigma}}_n$ for this data set results in accurate estimates such that the normalized estimation error is less than $0.06$ for all values of compression factor $\gamma$. The biased estimator $\widehat{\C}_n$ has approximately the same accuracy as $\widehat{\boldsymbol{\Sigma}}_n$ on this data set. The reason is because the bias term $\frac{\kappa}{m+1}\diag(\C_n)$ in equation~\eqref{eq:exp-C-hat-n} depends on both the kurtosis $\kappa$ and the structure of the underlying covariance matrix $\C_n$. Hence, the value of this bias term can be relatively small for some special structures on $\C_n$. However, our proposed estimator $\widehat{\boldsymbol{\Sigma}}_n$ provides an unbiased estimate of $\C_n$ without any restrictive assumptions on the data's structure.

Moreover, we examine the tradeoffs between computational savings and accuracy on the traffic data set. We report the running times to compute $\widehat{\boldsymbol{\Sigma}}_n$ and $\widehat{\C}_n$  normalized by the time spent to form the sample covariance matrix of the full data $\C_n$. As we see in Fig.~\ref{fig:traffic_cov_error}, the computational savings are proportional to $\gamma^2$. We also observe that $\widehat{\boldsymbol{\Sigma}}_n$ has roughly the same running time as $\widehat{\C}_n$.
\begin{figure}[!h]
	\centering
	\includegraphics[width=\textwidth]{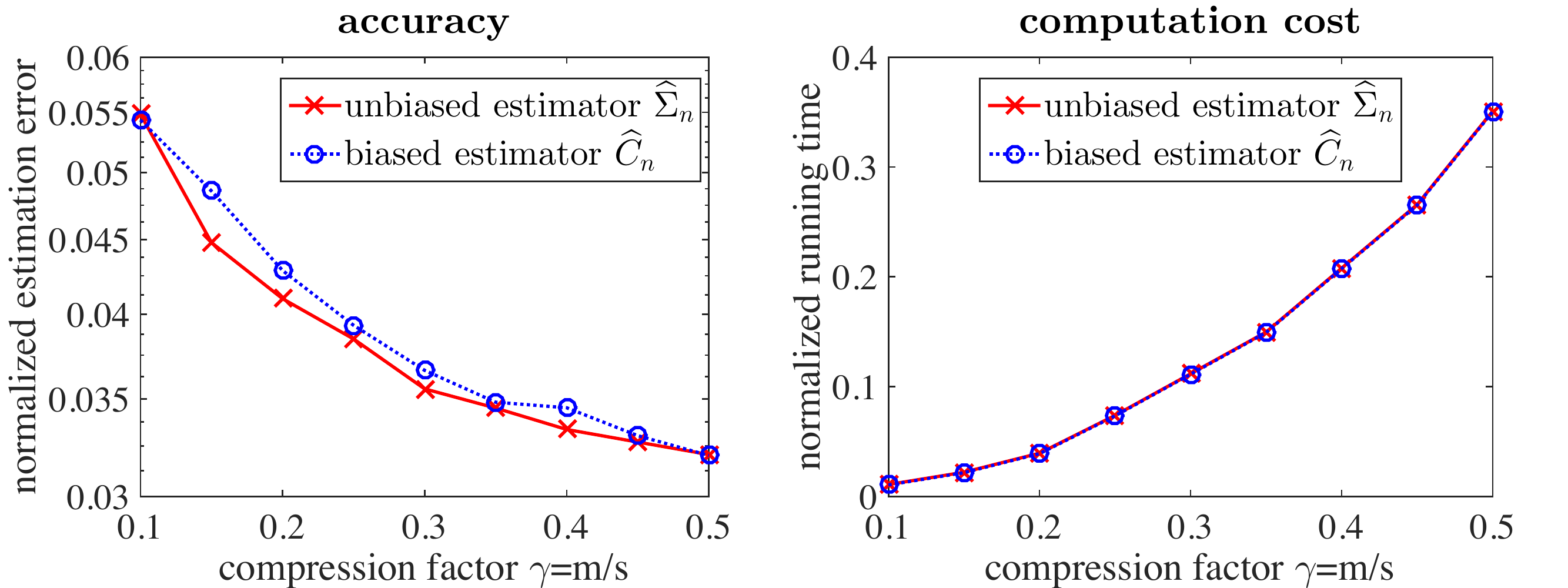}
	\caption{Accuracy and computation cost of the estimated covariance matrix from sparse random projections on the traffic data set for varying values of the compression factor $\gamma$. 
	}
	\label{fig:traffic_cov_error}
\end{figure}

In Fig.~\ref{fig:traffic_pc}, we provide a sample visual comparison of the first eigenvector of the underlying covariance matrix $\C_n$ and that estimated by our approach $\widehat{\boldsymbol{\Sigma}}_n$ for $\gamma=0.1$, $0.3$, and $0.5$. We see that the first eigenvectors of $\widehat{\boldsymbol{\Sigma}}_n$ and $\C_n$ are almost identical even for small values of the compression factor, e.g., $\gamma=0.1$. The first eigenvector of our covariance estimator $\widehat{\boldsymbol{\Sigma}}_n$ reveals the underlying structure from sparse random projections, which appears to be a traffic trend along the roadway.

This experiment typifies a real-world application of our approach. We showed that our proposed estimator $\widehat{\boldsymbol{\Sigma}}_n$ can be used to extract the covariance structure of the data from sparse random projections. At the same time, we achieve significant reduction of the computation cost proportional to $\gamma^2$, where $\gamma$ is the compression factor. 

\begin{figure}[!h]
	\centering
	\includegraphics[width=0.5\textwidth]{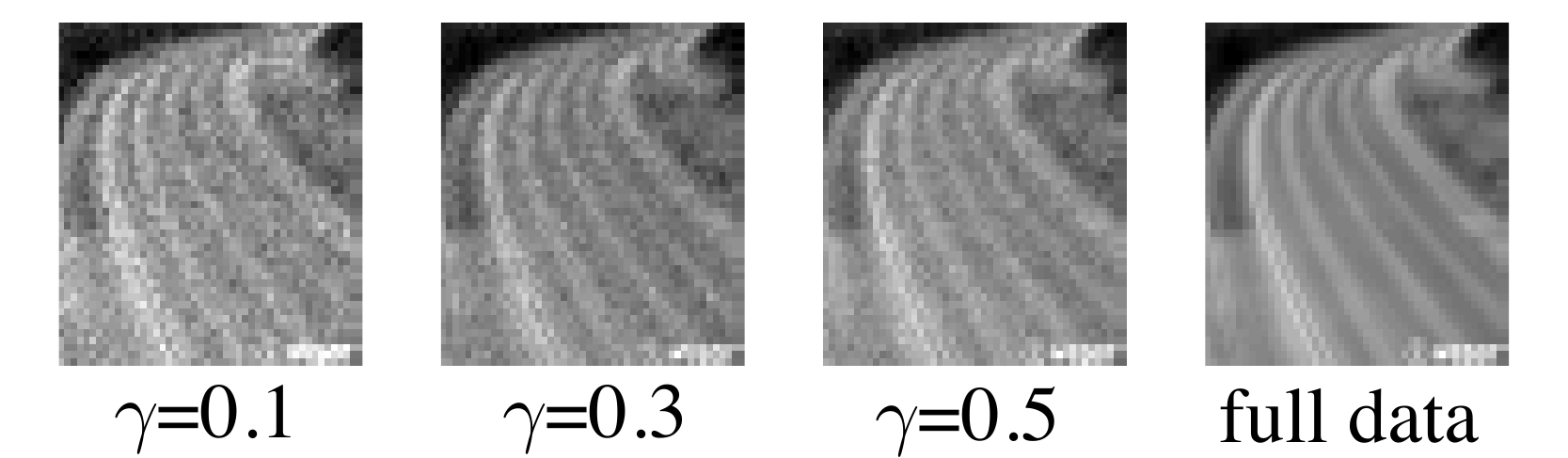}
	\caption{Visual comparison between the first eigenvector of the sample covariance matrix $\C_n$ (full data) and that estimated by our approach $\widehat{\boldsymbol{\Sigma}}_n$ on the traffic data set for the compression factor $\gamma=0.1$, $0.3$, and $0.5$.
	}
	\label{fig:traffic_pc}
\end{figure}
As one final note, we observe that our proposed covariance estimator $\widehat{\boldsymbol{\Sigma}}_n$ results in more accurate estimates on the traffic data set compared to the MNIST and gen4 data sets. To explain this, let us consider the stable rank $\beta:=\|\C_n\|_F^2/\|\C_n\|_2^2$,
where $\|\C_n\|_F$ represents the Frobenius norm of the sample covariance matrix $\C_n$. The parameter $\beta$ is a measure that gives us important information about
how spread out the eigenvalues of $\C_n$ are. In fact, a value of $\beta$ closer to $1$ indicates a faster rate of decay of the eigenvalues. The values of $\beta$ for each of the MNIST, gen4, and traffic data sets are $1.03$, $1.19$, and $1.00$ respectively. Therefore, the sample covariance matrix of the traffic data set has one  dominant eigenvector. 
Hence, the experimental results demonstrate that our proposed estimator $\widehat{\boldsymbol{\Sigma}}_n$ leads to more accurate estimates of $\C_n$ when the eigenvalues of $\C_n$ decay faster, i.e., $\beta$ is closer to $1$.  

\section{Conclusions}\label{Conclusion}
This paper presents an unbiased estimator to extract the covariance structure from low-dimensional random projections of data. As the main theoretical ingredients, our analysis holds for any random projection matrix consisting of i.i.d.~zero-mean entries with finite first four moments. In addition, our analysis does not require restrictive assumptions on the sample covariance matrix. Therefore, the present work extends prior work on compressive covariance estimation and provides a new analysis of covariance estimation from compressive measurements.

Furthermore, we have employed very sparse random projections to reduce the computation burden in high-dimensional data regimes. In fact, we showed that our approach can be used to estimate the covariance matrix, while achieving significant savings in computation cost. We have presented experimental results on several real-world data sets to show the accuracy of our proposed approach for varying values of the compression factor.

\subsection*{Acknowledgements}
I would like to thank Stephen Becker for his helpful comments and suggestions.

\newpage
\bibliographystyle{ieeetr}
\bibliography{PHD_Farhad}

\begin{thebibliography}{10}

\bibitem{DavenportCSSignal}
M.~Davenport, P.~Boufounos, M.~Wakin, and R.~Baraniuk, ``Signal processing with
  compressive measurements,'' {\em IEEE Journal of Selected Topics in Signal
  Processing}, vol.~4, pp.~445--460, 2010.

\bibitem{BigDataSigMagazine}
K.~Slavakis, G.~Giannakis, and G.~Mateos, ``Modeling and optimization for big
  data analytics: (statistical) learning tools for our era of data deluge,''
  {\em IEEE Signal Process. Mag.}, vol.~31, no.~5, pp.~18--31, 2014.

\bibitem{PerformanceLimitsCS}
T.~Wimalajeewa, H.~Chen, and P.~Varshney, ``Performance limits of compressive
  sensing-based signal classification,'' {\em IEEE Transactions on Signal
  Processing}, vol.~60, pp.~2758--2770, 2012.

\bibitem{ChangeAtia}
G.~Atia, ``Change detection with compressive measurements,'' {\em IEEE Signal
  Processing Letters}, vol.~22, no.~2, pp.~182--186, 2015.

\bibitem{CSCandes}
E.~Cand{\`e}s and M.~Wakin, ``An introduction to compressive sampling,'' {\em
  IEEE Signal Processing Magazine}, vol.~25, no.~2, pp.~21--30, 2008.

\bibitem{BlindCS_2010}
S.~Gleichman and Y.~Eldar, ``Blind compressed sensing,'' {\em IEEE Transactions
  on Information Theory}, vol.~57, no.~10, pp.~6958--6975, 2011.

\bibitem{StuderDL}
C.~Studer and R.~Baraniuk, ``Dictionary learning from sparsely corrupted or
  compressed signals,'' in {\em IEEE International Conference on Acoustics,
  Speech and Signal Processing (ICASSP)}, pp.~3341--3344, 2012.

\bibitem{Pourkamali_CKSVD}
F.~Pourkamali-Anaraki and S.~Hughes, ``Compressive {K-SVD},'' in {\em IEEE
  International Conference on Acoustics, Speech and Signal Processing
  (ICASSP)}, pp.~5469--5473, 2013.

\bibitem{Pourkamali_DL_SampTA}
F.~Pourkamali-Anaraki, S.~Becker, and S.~Hughes, ``Efficient dictionary
  learning via very sparse random projections,'' in {\em Sampling Theory and
  Applications (SampTA)}, pp.~478--482, 2015.

\bibitem{MallatPCA}
G.~Yu, G.~Sapiro, and S.~Mallat, ``Solving inverse problems with piecewise
  linear estimators: from gaussian mixture models to structured sparsity,''
  {\em IEEE Transactions on Image Processing}, vol.~21, no.~5, pp.~2481--2499,
  2012.

\bibitem{VerySparseRandom}
P.~Li, T.~Hastie, and K.~Church, ``Very sparse random projections,'' in {\em
  Proceedings of the 12th ACM SIGKDD international conference on Knowledge
  discovery and data mining}, pp.~287--296, 2006.

\bibitem{sindhwani2015structured}
V.~Sindhwani, T.~Sainath, and S.~Kumar, ``Structured transforms for
  small-footprint deep learning,'' in {\em Advances in Neural Information
  Processing Systems (NIPS)}, pp.~3088--3096, 2015.

\bibitem{PETRELS}
Y.~Chi, Y.~Eldar, and R.~Calderbank, ``{PETRELS}: Parallel subspace estimation
  and tracking by recursive least squares from partial observations,'' {\em
  IEEE Transactions on Signal Processing}, vol.~61, no.~23, pp.~5947--5959,
  2013.

\bibitem{kumar2013survey}
K.~Kumar, J.~Liu, Y.~Lu, and B.~Bhargava, ``A survey of computation offloading
  for mobile systems,'' {\em Mobile Networks and Applications}, vol.~18, no.~1,
  pp.~129--140, 2013.

\bibitem{azizyan2014subspace}
M.~Azizyan, A.~Krishnamurthy, and A.~Singh, ``Subspace learning from extremely
  compressed measurements,'' in {\em Asilomar Conference on Signals, Systems
  and Computers}, pp.~311--315, 2014.

\bibitem{park2014modal}
J.~Park, M.~Wakin, and A.~Gilbert, ``Modal analysis with compressive
  measurements,'' {\em IEEE Transactions on Signal Processing}, vol.~62, no.~7,
  pp.~1655--1670, 2014.

\bibitem{Fowler}
J.~Fowler, ``Compressive-projection principal component analysis,'' {\em IEEE
  Transactions on Image Processing}, vol.~18, no.~10, pp.~2230--2242, 2009.

\bibitem{NewBoundsCompressive}
A.~Kab{\'a}n, ``New bounds on compressive linear least squares regression,'' in
  {\em The 17-th International Conference on Artificial Intelligence and
  Statistics (AISTATS 2014)}, pp.~448--456, 2014.

\bibitem{Pourkamali_ICASSP_2014}
F.~Pourkamali-Anaraki and S.~Hughes, ``Efficient recovery of principal
  components from compressive measurements with application to {G}aussian
  mixture model estimation,'' in {\em IEEE Int. Conf. on Acoustics, Speech and
  Signal Processing (ICASSP)}, pp.~2332--2336, 2014.

\bibitem{Pourkamali_ICML_2014}
F.~Pourkamali-Anaraki and S.~Hughes, ``Memory and computation efficient {PCA}
  via very sparse random projections,'' in {\em Proceedings of the 31st
  International Conference on Machine Learning (ICML)}, pp.~1341--1349, 2014.

\end{thebibliography}

\end{document}